\documentclass[10pt]{article} 
\usepackage[preprint]{tmlr}

\usepackage{amsmath,amsfonts,bm}

\def\eqref#1{equation~\ref{#1}}

\def\1{\bm{1}}

\DeclareMathAlphabet{\mathsfit}{\encodingdefault}{\sfdefault}{m}{sl}
\SetMathAlphabet{\mathsfit}{bold}{\encodingdefault}{\sfdefault}{bx}{n}

\usepackage{hyperref}
\usepackage{url}

\usepackage{tikz}
\usetikzlibrary{positioning, calc, arrows.meta, backgrounds}
\usepackage{amsmath}
\usepackage{amssymb}

\definecolor{primaryblue}{RGB}{31, 119, 180}
\definecolor{secondaryblue}{RGB}{54, 144, 192}
\definecolor{accentgray}{RGB}{96, 125, 139}

\usepackage{hyperref}
\usepackage{url}
\usepackage{amsthm}

\usepackage{amsmath}
\usepackage{amssymb}
\usepackage{mathrsfs}
\usepackage{mathtools}
\usepackage{tikz-cd}

\newcommand{\cat}[1]{\mathsf{#1}}  
\newcommand{\Ob}{\mathsf{Ob}}      
\newcommand{\att}{\mathsf{att}} 
\newcommand{\AttP}{\mathsf{AttP}}  

\theoremstyle{definition} 
\newtheorem{definition}{Definition}[section]
\newtheorem{remark}{Remark}[section]

\theoremstyle{plain} 
\newtheorem{theorem}{Theorem}[section]

\title{Self-Attention as a Parametric Endofunctor: A Categorical Framework for Transformer Architectures}

\author{Charles O'Neill \\
School of Computing\\
Australian National University\\
\texttt{charles.oneill@anu.edu.au} \\
}

\def\month{MM}  
\def\year{YYYY} 
\def\openreview{\url{https://openreview.net/forum?id=XXXX}} 

\begin{document}

\maketitle

\begin{abstract}
Self-attention mechanisms have revolutionised deep learning architectures, yet their core mathematical structures remain incompletely understood. In this work, we develop a category-theoretic framework focusing on the linear components of self-attention. Specifically, we show that the query, key, and value maps naturally define a parametric 1-morphism in the 2-category \(\cat{Para}(\cat{Vect})\). On the underlying 1-category \(\cat{Vect}\), these maps induce an endofunctor whose iterated composition precisely models multi-layer attention. We further prove that stacking multiple self-attention layers corresponds to constructing the free monad on this endofunctor. For positional encodings, we demonstrate that strictly additive embeddings correspond to monoid actions in an affine sense, while standard sinusoidal encodings, though not additive, retain a universal property among injective (faithful) position-preserving maps. We also establish that the linear portions of self-attention exhibit natural equivariance to permutations of input tokens, and show how the ``circuits'' identified in mechanistic interpretability can be interpreted as compositions of parametric 1-morphisms. This categorical perspective unifies geometric, algebraic, and interpretability-based approaches to transformer analysis, making explicit the underlying structures of attention. We restrict to linear maps throughout, deferring the treatment of nonlinearities such as softmax and layer normalisation, which require more advanced categorical constructions. Our results build on and extend recent work on category-theoretic foundations for deep learning, offering deeper insights into the algebraic structure of attention mechanisms.
\end{abstract}

\section{Introduction}

Transformers have become the dominant architecture for a wide range of applications in natural language processing, computer vision, and beyond \citep{vaswani2017attention}. Their remarkable empirical success has outpaced our theoretical understanding, prompting numerous efforts to develop rigorous foundations that explain their expressive power, guide principled design choices, and deepen interpretability. Category theory — an abstract mathematical framework for describing structures and their relationships — has emerged as a powerful language for unifying disparate views of deep learning, from tensor programming paradigms \citep{abadi2016tensorflow, paszke2019pytorch} to symmetry-based approaches in geometric deep learning \citep{bronstein2021geometric, cohen2018spherical, weiler2023equivariant}.

Despite significant advances in understanding various aspects of transformers, the mathematical foundations of their core component — the self-attention mechanism — remain incompletely understood. While geometric approaches have illuminated symmetry properties and interpretability methods have revealed computational ``circuits,'' we lack a single unifying framework that shows precisely how these perspectives connect. This gap limits our ability to reason systematically about transformer architectures and their fundamental properties.

\begin{figure}[htb]
    \centering
    \begin{tikzpicture}[
        font=\sffamily\scriptsize,
        node distance=2cm,
        >=latex,
        shorten >=1pt,
        shorten <=1pt,
        on grid,
        auto
    ]
    \tikzstyle{block} = [
        draw=primaryblue,
        thick,
        rectangle,
        rounded corners=3pt,
        align=center,
        minimum width=1.8cm,
        minimum height=1.2cm,
        fill=white,
        text=black
    ]
    \tikzstyle{bigblock} = [
        draw=primaryblue,
        thick,
        rectangle,
        rounded corners=4pt,
        align=center,
        minimum width=2.8cm,
        minimum height=2.0cm,
        fill=white,
        text=black
    ]
    \tikzstyle{monoid} = [
        draw=primaryblue,
        thick,
        circle,
        minimum size=0.8cm,
        fill=white
    ]
    
    \begin{scope}
        \node[monoid] (monoidM) {$M$};
        \node[text width=2cm, align=center, above=0.8cm of monoidM, text=accentgray] {Positional Monoid};
        \node[block, right=3cm of monoidM] (posEnc) {
            Positional\\Encodings\\[3pt]
            \(\,\displaystyle p: \mathbf{BM}\!\to\!\cat{Vect}\,\)
        };
        
        \node[bigblock, 
            label={[align=center, text=accentgray]above:\textbf{Self-Attention}\\(Parametric Endofunctor)},
            right=4.5cm of posEnc,
            minimum height=3.8cm] (selfAttn) {};
            
        \node[block, minimum width=1.2cm, minimum height=0.9cm]
            (Q) at ($(selfAttn.center) + (0, 1.2cm)$) {\(\mathit{Q}\)};
        \node[block, minimum width=1.2cm, minimum height=0.9cm]
            (K) at ($(selfAttn.center) + (0, 0.0cm)$) {\(\mathit{K}\)};
        \node[block, minimum width=1.2cm, minimum height=0.9cm]
            (V) at ($(selfAttn.center) + (0, -1.2cm)$) {\(\mathit{V}\)};
            
        \node[block, right=3cm of selfAttn] (XnUpdated) {
            \(\mathrm{Att}(X^n)\)\\(New \\Embeddings)
        };
        
        \node[bigblock,
            label={[align=center, text=accentgray]above:Stacked Self-Attn\\(Free Monad on \(\,F\))},
            minimum width=2.6cm,
            minimum height=1.2cm,
            right=3.5cm of XnUpdated] (stacked) {
            \(\,F \circ F \circ \cdots\,\)
        };
    \end{scope}
    
    \draw[->, thick, primaryblue] (monoidM) -- (posEnc) 
        node[midway, above, text=accentgray] {\(\,m\mapsto p_m\,\)};
    \draw[->, thick, primaryblue] (posEnc) -- (selfAttn.west) 
        node[midway, above, text=accentgray] {Add \(\,p_i\) = \(\mathbf{X}^n\)};
    \draw[->, thick, primaryblue] (selfAttn) -- (XnUpdated);
    \draw[->, thick, primaryblue] (XnUpdated) -- (stacked) 
        node[midway, above, text=accentgray] {Repeat};
        
    \end{tikzpicture}
    \caption{Conceptual overview of how the monoid \(M\) (left) provides positional encodings \(\bigl(p:\mathbf{BM}\to \cat{Vect}\bigr)\) that are added to the embedding space \(\mathbf{X}^n\). The main self-attention block is formalised as a parametric endofunctor with learnable queries \((Q)\), keys \((K)\), and values \((V)\). Repeated application of this block yields the stacked self-attention layers, interpreted categorically as a free monad on \(F\).}
    \label{fig:self-attention-architecture}
\end{figure}
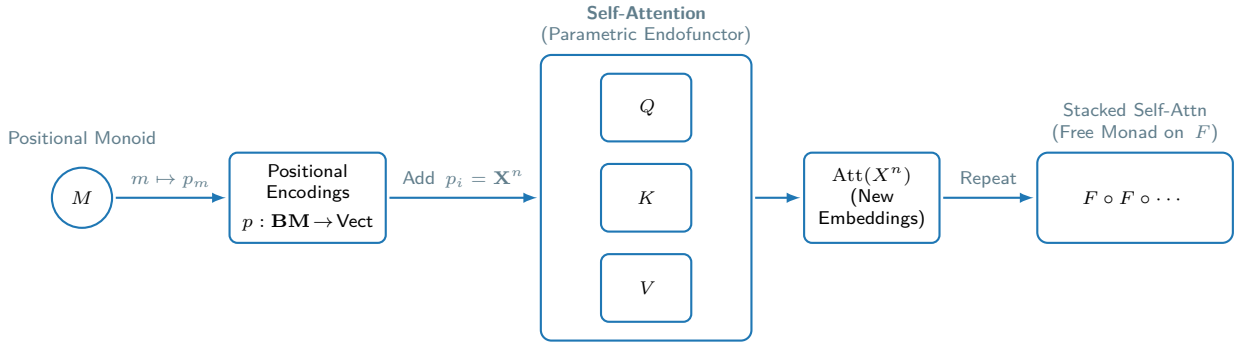

In this work, we take initial steps toward establishing a categorical framework for self-attention. We show that the linear components of self-attention can be viewed as a parametric 1-morphism in the 2-category \(\cat{Para}(\cat{Vect})\). Restricting to the underlying 1-category \(\cat{Vect}\), this same construction induces a parametric endofunctor whose repeated composition matches how multiple attention layers stack. Our perspective unifies mechanistic interpretability, symmetry arguments, and algebraic structures. Specifically:
\begin{enumerate}
\item \textbf{Self-Attention as a Parametric Endofunctor.} We formalise the linear portion of self-attention \((\mathbf{Q}, \mathbf{K}, \mathbf{V})\) as a parametric 1-morphism in \(\cat{Para}(\cat{Vect})\).  When viewed on the underlying category \(\cat{Vect}\), it acts like an endofunctor.  Stacking multiple layers corresponds to forming a free monad on that endofunctor.
\item \textbf{Positional Encodings in a Monoidal (Affine) Setting.} We treat strictly additive positional encodings as (affine) monoid actions on the embedding space, revealing how sequence indices can be functorially embedded.  In practice, the standard sinusoidal encodings used by most transformers are not strictly additive; we show they nonetheless form injective mappings of positions with a universal property among non-additive encodings.
\item \textbf{Equivariance and Symmetry in Self-Attention.} We establish that the linear part of self-attention exhibits natural equivariance with respect to permutations of input tokens. By connecting query, key, and value maps to group actions, we frame transformer layers within the language of equivariant neural networks, extending geometric deep learning insights to a broader algebraic setting.
\item \textbf{Connections to Circuit-Based Interpretability.} We demonstrate that our categorical viewpoint naturally encompasses circuit-based approaches \citep{elhage2021mathematical}: the minimal subgraphs identified in mechanistic interpretability correspond precisely to compositions of parametric morphisms. This provides a rigorous bridge between interpretability heuristics and formal category-theoretic structures.
\end{enumerate}

We focus exclusively on the linear components of self-attention. In real implementations, nonlinearities such as \(\mathrm{softmax}\) and layer norm alter attention patterns, but they lie outside the category \(\cat{Vect}\). By restricting to linear maps \(\mathbf{Q}, \mathbf{K}, \mathbf{V}\), we capture how parameters are shared and composed in a purely bilinear regime, sidestepping complications introduced by higher-order or non-linear operations. This choice enables a transparent exposition of the underlying 2-categorical framework while laying the groundwork for future extensions that may treat nonlinearities in more sophisticated categories (e.g.\ smooth manifolds or differential categories).

Our categorical perspective provides new mathematical tools for understanding core properties of transformers. By recasting mechanisms of attention in terms of monoid actions, endofunctors, and parametric morphisms, we unify different strands of deep learning theory — geometric, algebraic, categorical, and interpretability-based — within a cohesive framework. Although we do not claim a complete categorical treatment of transformers (non-linear operations and variable-length sequences remain beyond our current scope), we see this as an initial exploration aimed at inspiring future research and broader adoption of category-theoretic methods. Ultimately, we hope this approach will help illuminate design principles, clarify interpretability strategies, and advance the theoretical foundations of deep learning systems.

\section{Background}
\label{sec:background}

This section provides the necessary foundations for our categorical framework of self-attention. We begin by reviewing basic categorical concepts, then introduce the 2-category \(\cat{Para}(\cat{Vect})\) of parametric morphisms. We conclude by recalling how self-attention and positional encodings are formulated in standard transformers, along with a summary of related work.

\subsection{Category Theory Preliminaries}
A \emph{category} \(\mathcal{C}\) consists of a collection of \emph{objects} and \emph{morphisms} (arrows) between them, together with:
\begin{itemize}
    \item An associative composition rule \(\circ\) for composing morphisms.
    \item Identity morphisms \(\mathrm{id}_A\) on each object \(A\).
\end{itemize}
Formally:
\begin{definition}[Category]
A category \(\mathcal{C}\) consists of:
\begin{enumerate}
\item A collection of objects \(\Ob(\mathcal{C})\).
\item For each pair of objects \(A,B\), a set \(\mathrm{Hom}_\mathcal{C}(A,B)\) of morphisms.
\item For each object \(A\), an identity morphism \(\mathrm{id}_A \in \mathrm{Hom}_\mathcal{C}(A,A)\).
\item A composition operation \(\circ\) that is associative and respects identities.
\end{enumerate}
\end{definition}

A key category for our analysis is \(\cat{Vect}\), whose objects are finite-dimensional real vector spaces and whose morphisms are linear maps. Throughout, we also rely on the notion of \emph{endofunctors}:

\begin{definition}[Endofunctor]
An \emph{endofunctor} on a category \(\mathcal{C}\) is a functor \(F: \mathcal{C} \to \mathcal{C}\) that preserves composition and identities.
\end{definition}

Beyond ordinary categories, a \emph{2-category} allows us to handle \emph{morphisms between morphisms}, known as \emph{2-morphisms} \citep{leinster1999generalized, mac2013categories}. Informally, a 2-category has:
\begin{itemize}
    \item \textbf{Objects} (0-cells),
    \item \textbf{1-morphisms} between objects,
    \item \textbf{2-morphisms} between those 1-morphisms,
\end{itemize}
all of which obey suitable composition laws.

Our main object of study is \(\cat{Para}(\cat{Vect})\). Intuitively, \(\cat{Para}(\cat{Vect})\) captures the idea of a \emph{parametric} linear map, where each map depends on both ``parameters'' and ``data'' but remains linear in both. We define:

\begin{definition}[{\(\cat{Para}(\cat{Vect})\) with Tensor Products}]
\label{def:para-vect}
\(\cat{Para}(\cat{Vect})\) is the (strict) 2-category where:
\begin{itemize}
    \item \textbf{Objects}: finite-dimensional vector spaces \(X, Y, \dots\).
    \item \textbf{1-morphisms}: A 1-morphism from \(X\) to \(Y\) is a pair \((P, f)\), where
    \[
      P \in \cat{Vect}
      \quad\text{is a parameter space, and}\quad
      f : P \otimes X \;\longrightarrow\; Y
      \quad\text{is a linear map.}
    \]
      Here we use \(\otimes\) (tensor product) in lieu of the cartesian product \(\times\). This choice \textit{excludes} operations like copying or contracting parameters but aligns perfectly with the bilinear nature of weight matrices in linear neural layers.
    \item \textbf{2-morphisms}: Given two 1-morphisms \((P,f), (Q,g) : X \to Y\), a 2-morphism \(\rho: (P,f) \Rightarrow (Q,g)\) is a linear map \(\rho : Q \to P\) such that the following diagram commutes:
    \[
    \begin{tikzcd}[column sep=3em]
      Q \otimes X
      \arrow[r,"g"]
      \arrow[d,"\rho \otimes \mathrm{id}_X"']
      &
      Y
      \\
      P \otimes X
      \arrow[r,"f"']
      &
      Y \,.
    \end{tikzcd}
    \]
    Concretely, this “reparameterisation” \(\rho\) allows one parametric map to be turned into another via a linear relationship on the parameter spaces.
\end{itemize}
\end{definition}

\paragraph{Why tensor products?}  
Using \(\otimes\) rather than the more common cartesian product \(\times\) avoids copying/merging parameters. In linear attention layers, weights act bilinearly on \((\text{parameters}, \text{inputs})\), making \(\otimes\) the natural choice.  In typical “maximally general” parametric categories, one might use \(\times\) (cartesian), which admits diagonal maps and contraction, but does not match as neatly the strict bilinearity we see in self-attention’s query/key/value transformations.

\subsection{Self-Attention Mechanisms}

The transformer architecture's core component is the self-attention mechanism \citep{vaswani2017attention}. Given an input sequence $X = (x_1,\ldots,x_n)$ where each $x_i \in \mathbb{R}^d$, self-attention computes:
\begin{equation}
\text{Attention}(Q,K,V) = \text{softmax}\left(\frac{QK^T}{\sqrt{d_k}}\right)V
\end{equation}
where $Q = XW_Q$, $K = XW_K$, and $V = XW_V$ are linear projections of the input, with learned weight matrices $W_Q$, $W_K \in \mathbb{R}^{d \times d_k}$ and $W_V \in \mathbb{R}^{d \times d_v}$.

The linear components of self-attention can be formalised as morphisms in $\cat{Vect}$:
\begin{align*}
Q &: \mathbb{R}^d \rightarrow \mathbb{R}^{d_k} &
K &: \mathbb{R}^d \rightarrow \mathbb{R}^{d_k} &
V &: \mathbb{R}^d \rightarrow \mathbb{R}^{d_v}
\end{align*}

For a sequence of $n$ embeddings, these are applied componentwise:
\begin{align*}
Q^{(n)} &: (\mathbb{R}^d)^n \rightarrow (\mathbb{R}^{d_k})^n &
K^{(n)} &: (\mathbb{R}^d)^n \rightarrow (\mathbb{R}^{d_k})^n &
V^{(n)} &: (\mathbb{R}^d)^n \rightarrow (\mathbb{R}^{d_v})^n
\end{align*}

\subsection{Positional Encodings}

Standard transformer architectures add positional encodings to the input embeddings to capture sequence order \citep{chen2021simple}. For position $i$, the encoding $p_i \in \mathbb{R}^d$ is typically given by:
\begin{align*}
p_{i,2j} &= \sin(i/10000^{2j/d}) &
p_{i,2j+1} &= \cos(i/10000^{2j/d})
\end{align*}
The complete input to self-attention becomes $z_i = x_i + p_i$ for each position $i$.

\subsection{Related Work}

A growing body of research aims to provide mathematical foundations for transformers and related architectures, often drawing from frameworks in geometry, category theory, and logic. \textbf{Geometric deep learning} (GDL) has proven particularly fertile in this regard, emphasising principles of symmetry, equivariance, and representation theory \citep{bronstein2021geometric, cohen2018spherical, weiler2023equivariant}. Recent work has shown that embedding neural architectures into structured spaces or algebras can enforce desired symmetry properties, improving generalisation and interpretability \citep{andreeva2023metricspacemagnitudegeneralisation, farina2021symmetrydrivengraphneuralnetworks}.

Another important theoretical perspective emerges from work on parametric morphisms in category theory. While most general constructions of parametric categories use cartesian products to represent parameter spaces, allowing for copying and other non-linear behaviours \citep{johnson20212, cruttwell2022categorical}, our approach deliberately adopts a tensor-based framework. This choice, while more restrictive, aligns naturally with the linear-algebraic foundations of transformer architectures. Though this restriction to tensor products means we cannot capture certain non-linear or higher-order operations, it provides a cleaner theoretical framework for analysing the query/key/value decomposition in self-attention.

Within GDL, \textit{Geometric Algebra Transformers (GATr)} \citep{Brehmer2023GeometricAT, Haan2023EuclideanPC} are notable. These architectures parametrise network layers within projective or conformal geometric algebras, ensuring equivariance under E(3) and related symmetry groups. This geometric lens dovetails with earlier frameworks that target group equivariance using representation theory \citep{cohen2018spherical, weiler2023equivariant} and incorporate domain symmetries in message passing or attention mechanisms \citep{bronstein2021geometric}. Yet, these methods remain predominantly tied to specific symmetry groups or geometric structures, such as rotations in $\mathrm{SO}(3)$ or Clifford algebras, and do not provide a unified categorical formalism that seamlessly handles parameters, morphisms, and higher-level transformations.

Beyond geometric approaches, researchers have explored logic- and topos-theoretic perspectives. By embedding neural networks into categorical or logical frameworks, one can analyse expressive power and structural constraints. For example, \citet{Villani2024TheTO} frame transformers within a topos-theoretic setting, contrasting them with convolutional and recurrent networks that inhabit a more elementary ``first-order'' realm. Their findings complement earlier categorical perspectives on machine learning \citep{fong2019backprop, cruttwell2022categorical, gavranovicposition, capucci2021towards}, as well as works emphasising functorial abstractions of neural architectures \citep{johnson20212, lack20092, baydin2018automatic}. While these investigations articulate deep theoretical insights --- transformers emerge as ``higher-order reasoners'' living in a suitable topos completion --- they do not directly specify how to systematically dissect core building blocks like self-attention into clean categorical primitives that capture parameter sharing, monoidal structures, or compositional layers.

A related vein of research uses polynomial functors, invariant theory, and representation-theoretic characterisations to design or analyse equivariant neural architectures. For instance, \citet{Gregory2024LearningET} apply classical invariant theory and polynomial characterisations to build equivariant machine learning models that reflect structured symmetries, echoing earlier methods in equivariant networks \citep{cohen2018spherical, weiler2023equivariant, bronstein2021geometric}. These polynomial- and tensor-based approaches can recover or surpass known spectral methods and link to theoretical computer science perspectives. Still, they remain confined to particular group actions or algebraic constraints and do not generalise to a fully categorical setting that can, in principle, integrate arbitrary parameter spaces, linear maps, and multiple compositional layers.

Thus, prior work has produced a rich tapestry of theoretical tools --- geometric algebras for 3D equivariance \citep{Brehmer2023GeometricAT, Haan2023EuclideanPC, cohen2018spherical, bronstein2021geometric}, topos-theoretic analyses of expressivity \citep{Villani2024TheTO}, and invariant-theoretic insights into polynomial equivariance \citep{Gregory2024LearningET} --- but each line of inquiry remains specialised. They either target a narrow symmetry group, focus on logical expressivity without fully unpacking the structure of attention, or deal with polynomial invariants in isolation. The broader categorical approaches to neural networks \citep{fong2019backprop, cruttwell2022categorical, gavranovicposition, capucci2021towards} offer a potential unifying language but have yet to fully explicate transformer architectures in these terms. In contrast, we propose a \textit{categorical framework for self-attention as a parametric endofunctor} enriched into a 2-category, providing a unified lens that subsumes parameter sharing, positional encodings, and equivariance properties into a single, coherent algebraic structure.

\section{Self-Attention Components as Parametric 1-Morphisms}
\label{sec:self_attention_endofunctors}

In this section, we show how the linear parts of a single-head self-attention mechanism (query, key, and value maps) can be unified into a \emph{parametric 1-morphism} in the 2-category \(\cat{Para}(\cat{Vect})\).  From the perspective of the underlying \emph{1}-category \(\cat{Vect}\), this yields an \emph{endofunctor} that (informally) takes a vector-space input \(\,X\), applies Q/K/V to each token, and returns the triple of projected embeddings.  

We first recall that \(\cat{Para}(\cat{Vect})\) is a 2-category whose:
\begin{itemize}
\item \textbf{Objects} coincide with those of \(\cat{Vect}\) (finite-dimensional real vector spaces),
\item \textbf{1-morphisms} \(X\to Y\) are \((P,f)\) where \(P\) is a parameter vector space and \(f: P \otimes X \to Y\) is linear in both \(\,P\) and \(X\),
\item 2-morphisms \((P,f) \Rightarrow (P',f')\) are linear maps \(\rho : P' \to P\) making the diagram commute in the obvious reparameterisation sense.  
\end{itemize}

Below, we construct a single parametric 1-morphism \((\AttP,\att)\) corresponding to a single-head self-attention module, show it indeed composes as required in \(\cat{Para}(\cat{Vect})\), and then see how on the underlying category \(\cat{Vect}\) it induces an endofunctor that can be iterated.

\subsection{Self-Attention as a Parametric 1-Morphism}

A single-head self-attention mechanism (ignoring nonlinearities like softmax) involves three linear maps:
\[
  Q: \mathbb{R}^d \;\to\; \mathbb{R}^{d_k},
  \quad
  K: \mathbb{R}^d \;\to\; \mathbb{R}^{d_k},
  \quad
  V: \mathbb{R}^d \;\to\; \mathbb{R}^{d_v},
\]
parameterised by matrices \(W_Q, W_K, W_V\).  For a sequence length \(n\), we model the input space as \(X = (\mathbb{R}^d)^n\cong \mathbb{R}^{n\cdot d}\).  Then \(Q,K,V\) each act \emph{componentwise} on the tokens of \(X\).

\paragraph{Combining Q, K, and V into a Single Parametric Map.}
To capture all three maps at once within \(\cat{Para}(\cat{Vect})\), we do the following.

Let \(QP \cong \mathbb{R}^{(d_k\times d)}\) be the parameter space for \(W_Q\).  Then define a bilinear map
   \[
     q: QP \otimes (\mathbb{R}^d) \;\longrightarrow\; \mathbb{R}^{d_k},
     \quad
     (W_Q,\; x) \;\mapsto\; W_Q\cdot x.
   \]
When extended to sequences of length \(n\), we get \(q^{(n)} : QP\otimes X \to (\mathbb{R}^{d_k})^n\). Similarly, define \(KP, VP\) for the key/value matrices, yielding linear maps \(k^{(n)}, v^{(n)}\).Form the \emph{direct sum} of parameter spaces, 
   \[
     \AttP \;:=\; QP \;\oplus\; KP \;\oplus\; VP,
   \]
so that a global parameter \(\,\theta\in \AttP\) includes all three matrices \((W_Q, W_K, W_V)\).  

Define
   \[
     \att : \AttP \otimes X \;\longrightarrow\; 
       \bigl(\mathbb{R}^{d_k}\bigr)^n \;\otimes\; \bigl(\mathbb{R}^{d_k}\bigr)^n 
       \;\otimes\; \bigl(\mathbb{R}^{d_v}\bigr)^n
   \]
by sending \((\theta, x)\) to the triple
   \(\bigl(q^{(n)}(\theta,x),\,k^{(n)}(\theta,x),\,v^{(n)}(\theta,x)\bigr)\) in the product space.
Hence \((\AttP,\att)\) is a parametric 1-morphism \(X \to Y\) in \(\cat{Para}(\cat{Vect})\), where \(Y\) is 
\[
  Y := \bigl(\mathbb{R}^{d_k}\bigr)^n \otimes \bigl(\mathbb{R}^{d_k}\bigr)^n \otimes \bigl(\mathbb{R}^{d_v}\bigr)^n
  \;\simeq\; \mathbb{R}^{n\,d_k}\otimes \mathbb{R}^{n\,d_k}\otimes \mathbb{R}^{n\,d_v}.
\]
More precisely, if we prefer \(\oplus\) over \(\otimes\) for \(\cat{Vect}\)-products, we would adapt the same idea. The exact product vs.\ tensor details can be chosen according to how one encodes multi\-headed or multilinear aspects.

\paragraph{Composition in \(\cat{Para}(\cat{Vect})\).}
Recall that if we have two parametric morphisms
\[
  (P,f): X \to Y
  \quad\text{and}\quad
  (Q,g): Y \to Z,
\]
their composite is \(\bigl(Q\otimes P,\; h\bigr)\) with
\[
  h: (Q\otimes P)\otimes X
    \;\xrightarrow{\cong}\;
    Q \otimes (P\otimes X)
    \;\xrightarrow{\,Q\otimes f\,}\;
    Q \otimes Y
    \;\xrightarrow{\,g\,}\;
    Z.
\]
In other words, the parameter spaces combine via tensor (\(Q\otimes P\)), and we feed the output of \(f\) into \(g\).  For self-attention, stacking multiple heads or layers amounts to this \(\cat{Para}(\cat{Vect})\)-style composition of parametric morphisms, preserving the bilinear in parameters \(\times\) data viewpoint.

\begin{theorem}[Self-Attention as a Parametric 1-Morphism on \(\cat{Vect}\)]
\label{thm:selfattn-para-1morphism}
 
Let \(Q,K,V \colon \mathbb{R}^d \,\to\, \mathbb{R}^{d_k}\) (or \(\mathbb{R}^{d_v}\) for the value projection) be three linear maps representing the query, key, and value transformations of a single-head self-attention mechanism.  For a length-\(n\) sequence, the input space is \(X \,=\, (\mathbb{R}^d)^n \cong \mathbb{R}^{n \cdot d}\).  Then there exists a single \textbf{parametric 1-morphism}
\[
  (\AttP,\att)\;:\;X \,\longrightarrow\, Y
  \quad\text{in the 2-category}\;\cat{Para}(\cat{Vect}),
\]
where
\begin{enumerate}
\item \(\AttP\;=\;QP \,\oplus\, KP \,\oplus\, VP\) is a direct sum of parameter spaces for \(Q, K, V\).  
\item \(\att\colon \AttP \otimes X \,\to\, Y\) is a linear map (in both parameters and inputs) whose output encodes the triple \(\bigl(Q(x),K(x),V(x)\bigr)\) for each token \(x\in X\).  
\item \(Y\;\cong\;(\mathbb{R}^{d_k})^n\;\otimes\;(\mathbb{R}^{d_k})^n\;\otimes\;(\mathbb{R}^{d_v})^n\), i.e.\ the space of all query/key/value embeddings of an \(n\)-token input.
\end{enumerate}

Furthermore, this 1-morphism \((\AttP,\att)\) is \emph{stable under composition} in \(\cat{Para}(\cat{Vect})\). Concretely, stacking multiple attention layers corresponds to repeated composition of \((\AttP,\att)\) with similar parametric morphisms, matching how self-attention is practically stacked in transformer architectures.  

Finally, in the underlying 1-category \(\cat{Vect}\), \((\AttP,\att)\) induces an \emph{endofunctor} \(F\colon \cat{Vect} \to \cat{Vect}\), so that iterating self-attention amounts to iterating \(F\).
\end{theorem}

We provide a full proof in Appendix~\ref{app:self_attention_para_morphism}.

\subsection{From an Endofunctor to a Free Monad (Stacked Self-Attention)}

Since each single-head, linear self-attention yields an endofunctor \(F:\cat{Vect}\to\cat{Vect}\), \emph{stacking} multiple layers corresponds to repeatedly composing \(F\).  Formally, one obtains the \(\omega\)-chain
\[
  X \,\mapsto\, F(X) \,\mapsto\, F^2(X) \,\mapsto\, F^3(X) \,\mapsto\, \dots
\]
In category theory, collecting such iterates is precisely the \emph{free monad} construction on \(F\).  Concretely, the free monad \(\mathrm{Free}(F)\) is the colimit of all finite powers \(F^n\), with universal maps gluing them.  Thus:

\begin{theorem}[Stacking as a Free Monad]
 \label{thm:stacking-free-monad}  
Let \(F: \cat{Vect}\to \cat{Vect}\) be the endofunctor induced by a single-head, \emph{linear} self-attention mechanism, as constructed in Section~\ref{sec:self_attention_endofunctors}.  Then the colimit (in \(\cat{Vect}\)) of the sequence
\[
  \mathrm{id} \;\longrightarrow\; F \;\longrightarrow\; F^2 \;\longrightarrow\; \dots
\]
forms the \emph{free monad} on \(F\). In other words, \emph{stacking} self-attention layers is exactly building \(\mathrm{Free}(F)\).  

Equivalently, if we denote that colimit object by \(T\), then \(T\) admits a canonical monad structure \((T,\eta,\mu)\) extending \(F\), and any other monad \(M\) extending \(F\) factors uniquely through \(T\).
\end{theorem}

In practice, implementing multiple transformer layers means repeatedly applying Q/K/V (plus other operations) to produce deeper embeddings.  The free-monad viewpoint reveals a universal property: any other ``monad'' structure extending that same endofunctor must factor through \(\mathrm{Free}(F)\).  One can see this as the “most general way” to keep stacking self-attention blocks.

We give a detailed proof in Appendix~\ref{app:stacked_free_monad}, but intuitively the composition of parametric morphisms \(\,(Q\otimes P, h)\) ``unrolls'' indefinitely, and the colimit combines them into a single monadic object.

\section{Positional Encodings as (Potential) Monoid Actions}

Modern transformers use positional encodings to incorporate sequence order into an otherwise permutation-invariant architecture. In practice, a vector \(p_i\in X\) (often in \(\mathbb{R}^d\)) is added to each token embedding \(x_i\). From a strict linear\-algebraic perspective, adding \(p_i\) is an affine shift, not a purely linear transformation. In this section, we show that if these shifts satisfy \emph{strict additivity}, then they do form a (genuine) monoid action---but in the category of \emph{affine} transformations rather than strictly linear maps. We also compare strictly additive encodings to the popular \emph{sinusoidal} approach, which injectively labels positions but does not preserve the additive monoid structure.

\subsection{Strictly Additive Encodings as Affine Monoid Actions}

Let \(\,M\) be a monoid (e.g.\ \(\,( \mathbb{N}, + )\)) representing token positions.  Suppose we have a family \(\{\,p_m \in X : m \in M\}\) of vectors in a real vector space \(X\).  Define a map
\[
\alpha_m(x) \;=\; x \;+\; p_m
\]
for each \(m\in M\).  Strictly speaking, \(\alpha_m\colon X \to X\) is an \emph{affine} transformation (it does not fix the origin unless \(p_m=0\)).  We can treat \(X\) as an affine space and regard \(\alpha_m\) as a \emph{translation} by \(p_m\).  

\paragraph{Additivity Condition}

If we want \(\alpha_m\) to form a coherent monoid action by translations, we need:
\begin{enumerate}
\item \(p_0 = 0\), so \(\alpha_0\) is the identity shift, and
\item \(p_{m+m'} = p_m + p_{m'}\) for all \(m,m' \in M\), so that
   \[
     \alpha_{m+m'} \;=\; \alpha_m \,\circ\, \alpha_{m'}.
   \]
\end{enumerate}
This exactly says the family \(\{\,p_m\}\) behaves additively with respect to the monoid operation.

\begin{remark}[Affine vs.\ Linear]\label{rem:affine-vs-linear}
Although we often label \(\alpha_m\) as a “map” in \(\cat{Vect}\), it is \emph{not} strictly linear: 
\[
  \alpha_m(0) \;=\; p_m \;\neq\; 0, 
  \quad
  \alpha_m(x + y)\;\neq\;\alpha_m(x) + \alpha_m(y)\text{ in general.}
\]
Hence the more precise statement is that \(\alpha_m\) belongs in the category of \textbf{affine} transformations.  In an \emph{affine} category, such translations are valid morphisms.  Equivalently, if we fix an origin in \(X\), we can treat \(\alpha_m\) as “linear plus bias,” which is typical in machine learning.  
\end{remark}

\paragraph{Formal Monoid Action (Affine)}

If the shifts \(\,p_m\) satisfy 
\[
  p_{m+m'} \;=\; p_m + p_{m'},
  \quad
  p_0 = 0,
\]
then \(\{\alpha_m\}\) does indeed define an \emph{(affine) monoid action}.  From a purely categorical viewpoint, one can say it extends to a functor
\[
   \mathcal{B}M \;\longrightarrow\;\mathbf{Aff},
\]
where \(\mathbf{Aff}\) is the category of affine spaces and affine maps.  If you forcibly treat \(\alpha_m\) as a map in \(\cat{Vect}\), you must keep in mind it is affine in nature; it is “linear” only in the sense that \(\alpha_m(x) - \alpha_m(0) = x\).  

Consider the following example. A purely additive embedding \(p_m = m\cdot v\) for a fixed \(v\in X\) is the simplest instance. Then \(\alpha_m(x)=x + m\cdot v\) is exactly a translation by \(m\,v\). This scenario can capture a strictly linear progression of positions but may not match the typical sinusoidal or learned embeddings used in real transformers.

\subsection{Non-Additive (Sinusoidal) Encodings as Injective Labellings}

In practice, many standard or learned encodings \(\{\,p_i\}\) do not satisfy \(p_{i+i'}=p_i + p_{i'}\). The most famous is the \textbf{sinusoidal} scheme \citep{vaswani2017attention}, where
\[
  p_{i,2j} = \sin\!\bigl(\tfrac{i}{10000^{\,2j/d}}\bigr),
  \quad
  p_{i,2j+1} = \cos\!\bigl(\tfrac{i}{10000^{\,2j/d}}\bigr).
\]
Clearly, 
\(\sin(\alpha+\beta)\;\neq\;\sin(\alpha)+\sin(\beta)\), so this scheme is \emph{not} additive. Nonetheless, each position \(i\) is mapped to a distinct vector \(p_i\) (for typical index ranges and dimension \(d\)), which is enough for a transformer to distinguish positions.  

\paragraph{Faithful (Injective) Encodings}

We can still formalise such schemes if we only require \emph{injectivity}, ignoring the additive structure.  Concretely, let \(M\) be a monoid of positions (\(\mathbb{N}\) under \(+\), or finite \(\{0,\dots,n-1\}\)).  A map \(p\colon M\to X\) is called “faithful” or “injective” if \(p_m \neq p_{m'}\) whenever \(m\neq m'\).  We can regard it as a functor from the discrete category \(\mathcal{B}M\) to \(\cat{Vect}\) \emph{only in the trivial sense} that each element \(m\) is sent to a distinct point \(p_m\).  We do \emph{not} require \(\alpha_{m+m'} = \alpha_m\circ\alpha_{m'}\).  

One can define a category \(\mathcal{P}\) whose objects are pairs \((X, p)\) with \(p\) an injective map \(M\to X\), and whose morphisms are linear maps \(f: X\to Y\) preserving positions: \(f(p_m)=q_m\).  In such a category, \emph{any} strictly additive action is a \textit{special case}, but also any sinusoidal or learned embedding that injects positions is allowed.

\subsection{Rethinking ``Universality'' of Sinusoidal Encodings}

In earlier treatments, one might say that sinusoidal encodings are ``universal''.  However, strictly speaking, they are \emph{not} universal in the category of \(\mathbf{Aff}\) actions or \(\cat{Vect}\) actions, because they do not preserve the additive structure.  Instead, they can be seen as \(\textbf{initial objects}\) in the \textbf{looser category} \(\mathcal{P}\) of ``faithful position embeddings'' --- provided certain spanning assumptions hold (e.g.\ the sinusoidal vectors collectively generate \(\mathbb{R}^d\)).  

All that matters is that any injective embedding $(X,q)$ factors \emph{uniquely} through the sinusoidal embedding \((X^{\sin}, p^{\sin} )\) via a linear map \(f : X^{\sin} \to X\). This factorisation is guaranteed if \(\{p^{\sin(m)}\}\) forms a generating set for \(X^{\sin}\). Hence sinusoidal embeddings are universal among \emph{injective} encodings, not among strictly additive/affine ones.

For real transformers, being able to injectively label each position is often the only requirement.  Sinusoidal encodings do this in a phase-rich way that can help models recover relative distance.  But they do \emph{not} define a monoid action \(\alpha_{m+m'}=\alpha_m\circ\alpha_{m'}\), nor do they act linearly in \(\mathbf{Vect}\). A further discussion of these points is given in Appendix~\ref{app:sinusoidal_encodings}.

\section{Equivariance Properties of Self-Attention}

The linear components of self-attention mechanisms exhibit natural symmetries with respect to permutations of input tokens \citep{bronstein2021geometric, zaheer2017deep}. We can formalise these symmetries through the lens of group equivariance, demonstrating that the query, key, and value projections respect certain group actions on the sequence space.

\begin{theorem}
Let $G$ be a group acting on the index set $\{1, \ldots, n\}$ of input tokens, and consider the linear components of self-attention: the query, key, and value projections. Then these linear maps are equivariant with respect to the induced $G$-action on the embedding space. In other words, applying the group action before the attention maps is the same as applying the attention maps before the group action.
\end{theorem}

\begin{proof}
Suppose $G$ acts on the set $\{1,\ldots,n\}$ via a homomorphism $\sigma: G \to S_n$, so that for $g \in G$, $\sigma_g$ is a permutation of the indices. Consider the embedding space $X = \mathbb{R}^d$ for a single token, and $X^n \cong \mathbb{R}^{n \cdot d}$ for a sequence of $n$ tokens.

The group action on indices induces a linear representation $R_g: X^n \to X^n$ given by permuting the coordinates:
\[
    R_g(x_1,\ldots,x_n) = (x_{\sigma_g^{-1}(1)},\ldots,x_{\sigma_g^{-1}(n)}).
\]

Let $Q: X \to \cat{QSpace}$, $K: X \to \cat{KSpace}$, and $V: X \to \cat{VSpace}$ be the linear parts of self-attention. Each extends to sequences by applying them componentwise:
\[
    Q^{(n)}: X^n \to \cat{QSpace}^n, \quad K^{(n)}: X^n \to \cat{KSpace}^n, \quad V^{(n)}: X^n \to \cat{VSpace}^n.
\]

To show equivariance for $Q^{(n)}$, consider the following commutative diagram for each $g \in G$:
\[
    \begin{tikzcd}
        X^n \arrow[r, "Q^{(n)}"] \arrow[d, "R_g"'] & \cat{QSpace}^n \arrow[d, "R_g"] \\
        X^n \arrow[r, "Q^{(n)}"] & \cat{QSpace}^n
    \end{tikzcd}
\]

Following the top-down-left-right path:
\[
    Q^{(n)}(R_g(x_1,\ldots,x_n)) = Q^{(n)}(x_{\sigma_g^{-1}(1)},\ldots,x_{\sigma_g^{-1}(n)}) = (Q(x_{\sigma_g^{-1}(1)}),\ldots,Q(x_{\sigma_g^{-1}(n)})).
\]

Following the left-right-top-down path:
\[
    R_g(Q^{(n)}(x_1,\ldots,x_n)) = R_g(Q(x_1),\ldots,Q(x_n)) = (Q(x_{\sigma_g^{-1}(1)}),\ldots,Q(x_{\sigma_g^{-1}(n)})).
\]

Both paths yield the same result, showing $Q^{(n)} \circ R_g = R_g \circ Q^{(n)}$. The same argument applies identically to $K^{(n)}$ and $V^{(n)}$, giving:
\[
    K^{(n)} \circ R_g = R_g \circ K^{(n)}, \quad V^{(n)} \circ R_g = R_g \circ V^{(n)}.
\]

Thus, the linear parts of the self-attention mechanism are equivariant under the given $G$-action.
\end{proof}

The key insight is that these components act pointwise on the sequence elements, ensuring that the order in which we apply the group action and the linear maps does not affect the final result. The absence of coupling between different sequence positions in these linear transformations is what enables this clean equivariance property. In other words, the linear portions of self-attention maintain the symmetric structure of the input space, with any symmetry breaking coming solely from the subsequent nonlinear operations. 

\section{A Categorical Perspective on Transformer Circuits}

Mechanistic interpretability, and especially the ``circuits'' approach proposed by \citet{elhage2021mathematical}, frames a transformer layer as a sum of minimal computational subgraphs --- sometimes described as ``paths'' --- that route information from input tokens through queries, keys, values, and outputs to ultimately influence the model's logits. In an \emph{attention-only} transformer, each attention head can be dissected into:
\begin{itemize}
    \item \textbf{QK circuit}: A low-rank bilinear form $W_{QK} = W_Q^\top W_K$ that determines the attention pattern (i.e., which source tokens get attended to for a given destination token).
    \item \textbf{OV circuit}: A low-rank matrix $W_{OV} = W_O W_V$ that determines how the \emph{value} information from the source token modifies the destination token's representation.
\end{itemize}

\citet{elhage2021mathematical} consider composition of these ``subcircuits'' across multiple layers or heads. For instance, a token's representation can flow from token $t$ to token $t'$ via one head, and then from $t'$ to $t''$ via another, building up complex ``induction heads'' or ``skip-trigram'' patterns. Their ``path expansion'' approach accounts for every route by which an input token's embedding might affect the final logits.

\textbf{Our aim} is to show that each of the minimal circuits in the Elhage sense naturally \emph{is} a parametric morphism, and that composing circuits corresponds to the 2-categorical composition of parametric morphisms.

\subsection{Preliminaries: Factorising Transformer Components}
A hallmark of the circuits approach is the factorisation $W_{QK} = W_Q^\top W_K$ for queries and keys, and $W_{OV} = W_O\,W_V$ for outputs and values. Each is typically rank-limited by the dimension of an attention head, $d_{\text{head}}$, which is smaller than the full model dimension $d_{\text{model}}$. The result is a set of simple, bilinear or linear maps that are easy to interpret:
\begin{itemize}
    \item $W_{QK}$ shapes an $\mathbb{R}^d \times \mathbb{R}^d \to \mathbb{R}$ bilinear form, controlling attention scores.
    \item $W_{OV}$ shapes an $\mathbb{R}^d \to \mathbb{R}^d$ linear transformation, controlling how the chosen source token's ``value'' modifies the destination token's embedding.
\end{itemize}

We will interpret these factorised maps as parametric morphisms in $\cat{Para}(\cat{Vect})$, where each factor (like $W_Q, W_K, W_O, W_V$) becomes part of the parameter space for that morphism.

We summarise key notations:
\begin{itemize}
    \item \textbf{$Q\colon X \to \mathrm{QSpace}$} and \textbf{$K\colon X \to \mathrm{KSpace}$}: the linear transformations that map from the token-embedding space $X \simeq \mathbb{R}^{d_{\text{model}}}$ to smaller spaces $\mathrm{QSpace}$, $\mathrm{KSpace}$.
    \item \textbf{$V\colon X \to \mathrm{VSpace}$} and \textbf{$O\colon \mathrm{VSpace} \to X$}: similarly map from the embedding space to a head's value subspace, and then back to the embedding space.
    \item \textbf{$W_{QK} = W_Q^\top W_K$}, \textbf{$W_{OV} = W_O\,W_V$}: condensed ``circuit'' forms.
\end{itemize}

Hence, for each head, we will form $(P_{QK}, f_{QK})$ and $(P_{OV}, f_{OV})$ in $\cat{Para}(\cat{Vect})$, whose parameter spaces collectively encode $W_Q, W_K, W_O, W_V$.

\subsection{QK, OV, and MLP Blocks as Parametric Morphisms}

Consider a single-head attention mechanism ignoring non-linearities. The \emph{attention pattern} is determined by $Q\colon X \to Q_{h}$ and $K\colon X \to K_{h}$, typically realised by matrices $W_Q$ and $W_K$. However, we can gather these into a single parametric morphism:
\[
  (P_{QK},\, f_{QK}) \colon X^n \;\longrightarrow\; \mathbb{R}^{n \times n}
\]
where $P_{QK} \simeq \mathbb{R}^{d_\theta}$ is the parameter space (containing all entries of $W_Q, W_K$), and $f_{QK}$ is bilinear in the sense that each position's query interacts with each position's key. In a purely linear regime (without softmax), it is a linear map from $\,P_{QK} \otimes (X^n)\, \to \,\mathbb{R}^{n \times n}$. In practice, the $\mathrm{softmax}$ is then applied to these attention scores, but if we restrict to the linear portion, the QK mechanism is precisely $(P_{QK}, f_{QK})$.

Similarly, the ``OV'' circuit can be captured by a parametric morphism
\[
  (P_{OV},\, f_{OV}) \colon X^n \;\longrightarrow\; X^n,
\]
where $P_{OV}$ includes $W_O$ and $W_V$. Each position's \emph{value} is read from the input embedding by $W_V$, and then the result is mapped back to $X$ by $W_O$. If we consider that the \emph{same} parameters are used for each token position, we see that $(P_{OV}, f_{OV})$ is indeed a single parametric morphism replicated across the sequence length.

This clarifies \citet{elhage2021mathematical}'s statement that attention moves information from the source token to the destination token: the parametric map $(P_{OV}, f_{OV})$ \emph{lifts} the residual stream into a smaller subspace, modifies it linearly, and writes it back.

Finally, for a feedforward (MLP) block, we typically have:
\[
   \mathrm{MLP}(x) \;=\; W_{O}^{m}\, \sigma \bigl(W_{I}^{m}\,x \bigr),
\]
where $\sigma$ might be $\mathrm{GeLU}$ or $\mathrm{ReLU}$. Strictly, $\sigma$ is a nonlinearity, so if we want to remain in $\cat{Vect}$, we cannot handle $\sigma$ exactly as a linear morphism. However, if we restrict to the matrix multiplications alone, $\mathrm{MLP}$ again factors as a parametric object $(P_{\mathrm{MLP}}, f_{\mathrm{MLP}})$ plus a separate non-linear function.

In a more advanced \emph{smooth} or \emph{differential} category \citep{baez2011convenient}, we can incorporate $\sigma$ directly, but even in a linear approximation, $\mathrm{MLP}$ is parametric morphism-like \citep{blute2009cartesian, ehrhard2022differentials}. Mechanistic interpretability often focuses on the linear parts of the MLP (the input and output weights), so modelling them in $\cat{Para}(\cat{Vect})$ is straightforward.

\subsection{Compositional Structure: ``Paths'' as 2-Categorical Compositions}

A central idea of \citet{elhage2021mathematical} is to expand the final logits $\ell$ as a sum over all possible routes information may take through the network. In a single-layer setting, we might have a direct path (token embedding $\to$ unembedding) plus paths that go embedding $\to$ QK $\to$ OV $\to$ unembedding, and so on. Their formalism enumerates these paths, each of which contributes an additive term to the logits.

Viewed categorically, each path is exactly a sequential composition of the relevant parametric morphisms:
\[
  (P_{QK}, f_{QK}) \;\circ\; (P_{OV}, f_{OV}) \;\circ\; \cdots
\]
in the 2-category $\cat{Para}(\cat{Vect})$. The sum of paths corresponds to the fact that in standard transformer code, the \emph{output} from each composition is \emph{additively} combined in the residual stream or the final logits.

A \textbf{string diagram} for parametric morphisms conveniently represents each 1-morphism as a box with two wires: a horizontal wire for the data flow ($X \to Y$) and a vertical wire for the parameter flow ($P$) \citep{selinger2011survey, fong2019backprop}. Composing the horizontal wires captures function composition, and tensoring the vertical wires expresses that we combine parameter spaces via direct sum or product when we stack heads or layers.

Hence, the circuit for a single attention head is a single box labeled ``$\mathrm{QK}/\mathrm{OV}$'' with a vertical line for the parameter space $P_{QK} \oplus P_{OV}$. A stacked attention layer or multi-head scenario arises by 2-categorical composition (for reparameterisation) and monoidal product (for parallel heads). This is shown in Figure \ref{fig:string-diagram-circuits}.

\begin{figure}[htb]
\centering
\begin{tikzpicture}[
    x=2.2em,         
    y=2.2em,         
    baseline=(current bounding box.center),
    font=\small,
    >=latex
]

\tikzstyle{data} = [->, thick, black]
\tikzstyle{param} = [->, thick, color=gray!70]
\tikzstyle{morphism} = [
    draw,
    thick,
    fill=blue!15,
    rectangle,
    minimum width=4em,
    minimum height=2.5em,
    align=center
]

\node (Xleft) at (0,0) {$X^n$};   
\node (Xmid)  at (4,0) {$X^n$};   
\node (Xright) at (8,0) {$X^n$};  

\draw[data] (Xleft) -- (Xmid);
\draw[data] (Xmid) -- (Xright);

\node at ($(Xleft)!0.5!(Xmid) + (0,1.8)$) (PQK) {$P_{QK}$};   
\draw[param] (PQK) -- ($(Xleft)!0.5!(Xmid)$);                

\path let \p1 = ($(Xleft)!0.5!(Xmid)$) in
      node[morphism, anchor=center] (QKbox) at (\x1,0) {QK};

\node at ($(Xmid)!0.5!(Xright) + (0,1.8)$) (POV) {$P_{OV}$};  
\draw[param] (POV) -- ($(Xmid)!0.5!(Xright)$);               

\path let \p2 = ($(Xmid)!0.5!(Xright)$) in
      node[morphism, anchor=center] (OVbox) at (\x2,0) {OV};

\node[below=0.4em of QKbox, align=center] {\(\bigl(P_{QK}, f_{QK}\bigr)\)};
\node[below=0.4em of OVbox, align=center] {\(\bigl(P_{OV}, f_{OV}\bigr)\)};

\node[below=3.0em of Xmid, text=gray!90!black, align=center] 
  {\(\text{Composition in } \cat{Para}(\cat{Vect})\)};

\end{tikzpicture}
\caption{A string diagram illustrating two parametric morphisms (QK and OV) composed sequentially in
\(\cat{Para}(\cat{Vect})\). Each morphism has a horizontal wire (data flow) and a vertical wire
(parameter space). The output of the first morphism (QK) feeds into the second (OV).}
\label{fig:string-diagram-circuits}
\end{figure}
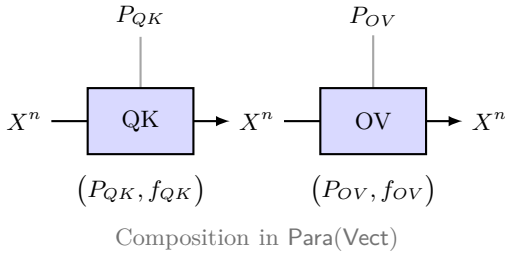
In standard linear algebra, if multiple parallel transformations exist from $X$ to $Y$, we sum their outputs. Mechanistic interpretability enumerates these parallel transformations --- like heads in an attention layer --- as the union of circuits. The 2-category $\cat{Para}(\cat{Vect})$ can model parallel heads as separate 1-morphisms whose outputs are added post-composition. Thus, each term in the Elhage path expansion is literally a path in the 2-category sense: an ordered composition of parametric morphisms. Summation arises from the additive nature of the residual stream or final logit computation.

\subsection{Examples: Stacked Self-Attention \& Circuit Factorisation}

A toy one-layer, single-head transformer ignoring nonlinearities is:
\[
  \ell(x) \;=\; W_U \,W_E \;+\; W_U\,W_{OV}\,W_E \;\otimes\; A\bigl(W_{QK}\bigr),
\]
where $W_E$ is the token embedding, $W_U$ is the unembedding, $W_{QK} = W_Q^\top W_K$, $W_{OV} = W_O\,W_V$, and $A\bigl(W_{QK}\bigr)$ is the attention mixing across tokens. In the parametric lens, we define:
\[
  (P_{QK}, f_{QK}) \quad\text{and}\quad (P_{OV}, f_{OV})
\]
respectively for QK and OV. The direct path $W_U\,W_E$ is just another linear map (possibly also represented as a parametric morphism with parameters $(P_{\mathrm{direct}}, f_{\mathrm{direct}})$). Thus, the decomposition into circuits---``bigram'' direct path plus skip-trigram (OV--QK) path---matches the sum of parametric morphisms.

\citet{elhage2021mathematical} highlight \emph{induction heads} in a two-layer transformer: the second-layer head's attention pattern depends on a \emph{previous token head} from the first layer, enabling repetition or ``copying'' of patterns. This is a composition in $\cat{Para}(\cat{Vect})$ of:
\begin{enumerate}
    \item The parametric morphism for first-layer QK/OV, $(P^1_{\mathrm{QK}}, f^1_{\mathrm{QK}})$ and $(P^1_{\mathrm{OV}}, f^1_{\mathrm{OV}})$.
    \item A reparameterisation that passes the first layer's output subspace into the second layer's input parameter space.
    \item The second-layer QK/OV morphisms, $(P^2_{\mathrm{QK}}, f^2_{\mathrm{QK}})$ and $(P^2_{\mathrm{OV}}, f^2_{\mathrm{OV}})$.
\end{enumerate}

Hence, an ``induction head circuit'' is a two-stage composition of parametric morphisms that unify Q--K composition across the layers. The same viewpoint extends to multi-layer networks, with more compositional ``virtual heads'' from repeated application.

\subsection{Interpreting Higher-Order ``Virtual Heads''}
In the circuits framework, a ``virtual head'' arises when one head's output is fed as input to another head in a deeper layer or the same layer. In $\cat{Para}(\cat{Vect})$ terms, this is exactly the composition of 1-morphisms $(P_1,f_1) \circ (P_2,f_2)$. If repeated, one obtains higher-order compositions $(P_n,f_n) \circ \cdots \circ (P_1,f_1)$. Mechanistic interpretability's enumerations of these ``virtual attention heads'' match exactly the chain of parametric morphisms under 2-categorical composition.

\bigskip 

By recasting Q, K, V, O, and MLP blocks as parametric morphisms in $\cat{Para}(\cat{Vect})$, we obtain a crisp perspective:
\begin{enumerate}
    \item \textbf{Circuit Blocks $\leftrightarrow$ Parametric 1-Morphisms.} Each minimal circuit (QK, OV, MLP input, MLP output) is a map $(P,f)$.
    \item \textbf{Paths $\leftrightarrow$ Composed Morphisms.} The chain of blocks that modifies a token's embedding is composition in the 2-category.
    \item \textbf{Weight Tying $\leftrightarrow$ 2-Morphisms.} If different heads share parameters or reparameterise them, that corresponds to 2-morphisms $(P' \to P)$.
\end{enumerate}

We have thus shown that the \textbf{circuits} view of transformers---decomposing them into minimal QK and OV blocks that route information between tokens---maps naturally to a \textbf{parametric 2-categorical} framework. In $\cat{Para}(\cat{Vect})$, each \textbf{QK} and \textbf{OV} pair is a \textbf{parametric 1-morphism}, with a parameter space for the learned weights. Stacking or chaining these blocks across layers yields \textbf{compositions} of parametric morphisms. The paths enumerated in mechanistic interpretability precisely match the sum (in the residual stream) of these compositional chains.

\section{Discussion}

Our categorical framework for transformers provides new mathematical tools for analysing these architectures while establishing formal connections between their components and classical algebraic structures. By demonstrating that \emph{linear} self-attention maps form a \emph{parametric 1-morphism} in the 2-category \(\cat{Para}(\cat{Vect})\), which induces an endofunctor on the underlying category \(\cat{Vect}\), we have shown how (strictly additive) positional encodings can be interpreted as (affine) monoid actions, how these structures exhibit natural equivariance properties, and how they align with circuit-based interpretability approaches. Together, these results yield a unified algebraic treatment of core transformer mechanisms.

The value of this categorical abstraction lies in revealing the fundamental mathematical structures underlying transformers. By identifying self-attention components as morphisms in \(\cat{Para}(\cat{Vect})\), we see parameter sharing and weight tying emerge naturally from 2-categorical coherence conditions. Our characterisation of \emph{stacked} self-attention as a free monad further illuminates the algebraic nature of deep transformer architectures. This perspective generalises previous geometric deep learning insights \citep{bronstein2021geometric}, which focus primarily on group equivariance, to a broader algebraic setting that includes monoid actions and parametric functors.

A key strength of our framework is how it naturally encompasses and extends the circuits formalism for transformer interpretability \citep{elhage2021mathematical}. Where circuits decompose attention heads into computational subgraphs, our categorical treatment reveals these circuits as parametric morphisms in \(\cat{Para}(\cat{Vect})\). The breakdown of attention into query, key, and value projections aligns closely with the circuit-based separation into QK and OV circuits, while the categorical perspective offers additional structure via composition laws and coherence conditions. This creates a formal bridge between interpretability heuristics and universal constructions in category theory, pointing to new ways of analysing and understanding transformer behaviour.

Transformers are particularly well-suited to categorical analysis because their core operations---parameter sharing, compositional structure, sequence indexing, and computational pathways---map naturally to concepts like natural transformations, \emph{affine} monoid actions, endofunctors, and parametric 1-morphisms. This alignment supports the broader vision of category theory as a unifying framework for deep learning \citep{gavranovicposition}, while demonstrating its concrete applicability to a central model class.

Although we highlight how \emph{strictly additive} positional encodings can be captured via (affine) monoid actions, real implementations often rely on \emph{sinusoidal} or learned encodings that are not strictly additive. As shown above, such sinusoidal encodings can nevertheless be understood as \emph{faithful functors} from positions into a vector space, preserving distinctness of positions without satisfying $p_{m+m'} = p_m + p_{m'}$. We reconcile this by treating standard sinusoidal schemes as injective mappings that still admit a universal property among faithful (non-additive) encodings, even if they do not reflect a strictly linear or affine shift.

\paragraph{Scope and Future Directions.}
Despite these unifying insights, our treatment focuses on the \emph{linear} skeleton of self-attention. Key nonlinearities---such as \(\mathrm{softmax}\) in attention, or activation functions and layer norms in MLP blocks---lie outside the category \(\cat{Vect}\). Incorporating these into a strictly categorical perspective will require moving beyond \(\cat{Vect}\) to structures like \emph{smooth manifolds}, \emph{convex categories}, or \emph{differential categories}, as discussed in Appendix~\ref{app:nonlinearities}. We also assume fixed sequence lengths, although practical transformers handle variable-length inputs. In addition, we have only briefly touched on the interplay between self-attention and MLP layers, which might be more fully explained via higher-level 2-monads or enriched categorical constructions \citep{street1972formal, kelly2002categories}.

We believe pursuing these directions will yield deeper insights and more powerful abstractions of the transformer pipeline. Indeed, bridging from purely linear parametric morphisms to differential categories holds promise for explaining and guiding gradient-based learning, while an extended categorical view of MLP blocks could unify the entire architecture in a single coherent framework. Ultimately, our work provides a lens on attention via parametric endofunctors and monoid actions, inviting future investigations that integrate the remaining practical and theoretical dimensions.

\newpage

\bibliography{main}
\bibliographystyle{tmlr}

\appendix

\section{Proof of Self-Attention as a Parametric Morphism}
\label{app:self_attention_para_morphism}

\begin{theorem}[Self-Attention as a Parametric 1-Morphism on \(\cat{Vect}\)]
\label{thm:selfattn-para-1morphism}
 
Let \(Q,K,V \colon \mathbb{R}^d \,\to\, \mathbb{R}^{d_k}\) (or \(\mathbb{R}^{d_v}\) for the value projection) be three linear maps representing the query, key, and value transformations of a single-head self-attention mechanism.  For a length-\(n\) sequence, the input space is \(X \,=\, (\mathbb{R}^d)^n \cong \mathbb{R}^{n \cdot d}\).  Then there exists a single \textbf{parametric 1-morphism}
\[
  (\AttP,\att)\;:\;X \,\longrightarrow\, Y
  \quad\text{in the 2-category}\;\cat{Para}(\cat{Vect}),
\]
where
\begin{enumerate}
\item \(\AttP\;=\;QP \,\oplus\, KP \,\oplus\, VP\) is a direct sum of parameter spaces for \(Q, K, V\).  
\item \(\att\colon \AttP \otimes X \,\to\, Y\) is a linear map (in both parameters and inputs) whose output encodes the triple \(\bigl(Q(x),K(x),V(x)\bigr)\) for each token \(x\in X\).  
\item \(Y\;\cong\;(\mathbb{R}^{d_k})^n\;\otimes\;(\mathbb{R}^{d_k})^n\;\otimes\;(\mathbb{R}^{d_v})^n\), i.e.\ the space of all query/key/value embeddings of an \(n\)-token input.
\end{enumerate}

Furthermore, this 1-morphism \((\AttP,\att)\) is \emph{stable under composition} in \(\cat{Para}(\cat{Vect})\). Concretely, stacking multiple attention layers corresponds to repeated composition of \((\AttP,\att)\) with similar parametric morphisms, matching how self-attention is practically stacked in transformer architectures.  

Finally, in the underlying 1-category \(\cat{Vect}\), \((\AttP,\att)\) induces an \emph{endofunctor} \(F\colon \cat{Vect} \to \cat{Vect}\), so that iterating self-attention amounts to iterating \(F\).
\end{theorem}

\begin{proof}
We will demonstrate this result in three main parts, showing first the construction of the parametric 1-morphism, then verifying its stability under composition, and finally examining the induced endofunctor on the underlying 1-category.

Let us begin with the construction of the parametric 1-morphism $(\AttP,\att)$. Consider first the parameter spaces for $Q,K,V$. Let $Q: \mathbb{R}^d \to \mathbb{R}^{d_k}$ be given by a matrix $W_Q$ of dimension $d_k\times d$. We define a parameter space $QP \cong \mathbb{R}^{d_k \times d}$, where each element of $QP$ corresponds to a possible choice of $W_Q$. Similarly, we define
\[
  KP \;\cong\; \mathbb{R}^{d_k \times d}, 
  \quad
  VP \;\cong\; \mathbb{R}^{d_v \times d},
\]
for the key and value maps. These three parameter spaces combine into a single direct sum:
\[
  \AttP 
  \;=\;
  QP \;\oplus\; KP \;\oplus\; VP.
\]
Thus an element $\theta \in \AttP$ is a triple $(W_Q, W_K, W_V)$.

For the bilinear map $\att\colon \AttP \otimes X \to Y$, we consider the input space $X = (\mathbb{R}^d)^n \cong \mathbb{R}^{n \cdot d}$, where a point $x\in X$ is a concatenation $(x_1,\dots,x_n)$, each $x_i\in \mathbb{R}^d$. The output space $Y$ must produce $(QX, KX, VX)$ for the entire sequence, naturally leading to:
\[
  Y 
  \;=\;
  (\mathbb{R}^{d_k})^n \;\otimes\; (\mathbb{R}^{d_k})^n \;\otimes\; (\mathbb{R}^{d_v})^n.
\]

We define $\att: \AttP\otimes X \longrightarrow Y$ by letting $\att(\theta,\,x)$ be the triple
\[
   \bigl(\,Q(x),\,K(x),\,V(x)\bigr)
   \;\in\;
   (\mathbb{R}^{d_k})^n
     \times
   (\mathbb{R}^{d_k})^n
     \times
   (\mathbb{R}^{d_v})^n,
\]
where $Q(x)$ denotes applying $W_Q$ to each token in $x$. In purely linear terms, if $\theta = (W_Q,W_K,W_V)$ and $x=(x_1,\dots,x_n)$, then
\[
  \att(\theta,x)
  \;=\;
  \Bigl(\,(W_Q\,x_1,\dots,W_Q\,x_n),\,\,(W_K\,x_1,\dots,W_K\,x_n),\,\,(W_V\,x_1,\dots,W_V\,x_n)\Bigr).
\]

To verify that $\att$ is a morphism in $\cat{Vect}$, we must check its linearity in both the parameter space $\AttP$ and the input space $X$. For linearity in $\theta$, fix $x$ and suppose $\theta_1,\theta_2\in \AttP$ and $\alpha\in \mathbb{R}$. Then
\[
  \att(\,\theta_1 + \alpha\,\theta_2,\; x)
  \;=\;
  \bigl(Q(x),\,K(x),\,V(x)\bigr)
  \quad\text{with}\;Q,K,V\;\text{the linear maps from}\;\theta_1 + \alpha\,\theta_2.
\]
Since matrix addition and scalar multiplication combine linearly in $(W_Q,W_K,W_V)$, each output coordinate $W_Q x_i$ depends linearly on $\theta_1 + \alpha\,\theta_2$, and the same applies for $W_K$ and $W_V$. Hence
\[
  \att(\theta_1 + \alpha\,\theta_2,\; x)
  \;=\;
  \att(\theta_1,\; x)
  \;+\;
  \alpha\,\att(\theta_2,\; x).
\]

For linearity in $x$, fix $\theta$ and suppose $x,\,x'\in X$ and $\beta\in \mathbb{R}$. Then
\[
  \att(\,\theta,\; x + \beta\,x')
  \;=\;
  \bigl(W_Q(x + \beta\,x'),\,W_K(x + \beta\,x'),\,W_V(x + \beta\,x')\bigr).
\]
Each of these terms is linear in $x$ by definition of matrix multiplication. Therefore,
\[
  \att(\theta,\,x + \beta\,x')
  \;=\;
  \att(\theta,\,x)
  \;+\;
  \beta\,\att(\theta,\,x').
\]

These properties confirm $\att\colon \AttP\otimes X\to Y$ is a linear map in $\cat{Vect}$, establishing that $(\AttP,\att)$ is indeed a valid 1-morphism in $\cat{Para}(\cat{Vect})$.

Moving to the verification of stability under composition, suppose we have another parametric morphism $(\AttP',\att')\colon Y\to Z$ in $\cat{Para}(\cat{Vect})$. The composition of $(\AttP,\att)$ and $(\AttP',\att')$ yields a well-defined parametric 1-morphism
\[
  \bigl(\,\AttP' \otimes \AttP,\; h\,\bigr)\;:\;X\;\to\;Z
\]
whose action feeds the output of $\att$ into $\att'$. By definition of composition in $\cat{Para}(\cat{Vect})$, the new parameter space is $\,\AttP'\otimes\AttP$, where $\theta'\in\AttP'$ parameterizes the second morphism $\att'$, and $\theta\in\AttP$ parameterizes the first morphism $\att$. 

We define the composite map $h$ as:
\[
  h: (\AttP'\otimes \AttP)\;\otimes\;X
    \;\xrightarrow{\cong}\;
    \AttP' \,\otimes\,(\AttP \otimes X)
    \;\xrightarrow{\;\AttP'\otimes\att\;}
    \AttP' \,\otimes\, Y
    \;\xrightarrow{\;\att'\;}
    Z.
\]

This process first "rebracketes" $(\theta'\otimes\theta)\otimes x$ as $\theta'\otimes(\theta\otimes x)$, then applies $\theta\otimes x \mapsto \att(\theta,x)\in Y$, resulting in $\theta'\otimes(\att(\theta,x))$ as input to $\att'$. Since both $\att$ and $\att'$ are bilinear, this composite map $h$ is bilinear in $\theta'\otimes\theta$ and $x$, making it a linear map in $\cat{Vect}$. Thus $(\AttP'\otimes\AttP,\;h)$ is a valid 1-morphism $X\to Z$ in $\cat{Para}(\cat{Vect})$.

Finally, we examine how $(\AttP,\att)$ induces an endofunctor $F\colon \cat{Vect}\to\cat{Vect}$ on the underlying 1-category. For an object $X\in \cat{Vect}$, we set
\[
  F(X) 
  \;=\; 
  Y(X)
  \;\cong\;
  \bigl(\mathbb{R}^{d_k}\bigr)^n 
      \;\otimes\;
   \bigl(\mathbb{R}^{d_k}\bigr)^n
      \;\otimes\;
   \bigl(\mathbb{R}^{d_v}\bigr)^n.
\]

Given a linear map $f\colon X\to X'$ in $\cat{Vect}$, we define $F(f)\colon F(X)\to F(X')$ by:
\[
  F(f):
  \;\AttP\otimes X
  \;\xrightarrow{\;\,\mathrm{id}_{\AttP}\otimes f\,}\;
  \AttP\otimes X'
  \;\xrightarrow{\;\att\,}\;
  Y(X').
\]
Here, $\mathrm{id}_{\AttP}\otimes f$ applies $f$ to the input part of $(\theta,x)$, and $\att$ maps $(\theta,\;f(x))$ to the triple of projected embeddings in $(\mathbb{R}^{d_k})^n \otimes (\mathbb{R}^{d_k})^n \otimes (\mathbb{R}^{d_v})^n$.

For functoriality, observe that if $f = \mathrm{id}_X$, then $F(f) = \mathrm{id}_{F(X)}$ by linearity. Moreover, if $g\circ f\colon X\to X''$ is a composition, then 
\[
  F(g\circ f)
  \;=\; \att\circ (\mathrm{id}_{\AttP}\otimes (g\circ f))
  \;=\; \att\circ \bigl[\mathrm{id}_{\AttP}\otimes g\bigr] \circ \bigl[\mathrm{id}_{\AttP}\otimes f\bigr]
  \;=\; F(g)\circ F(f).
\]
This confirms that $F$ strictly preserves composition and identities, establishing it as a bona fide endofunctor on $\cat{Vect}$.
\end{proof}

\section{Proof of Stacked Self-Attention as a Free Monad}
\label{app:stacked_free_monad}

In this appendix, we outline why repeatedly composing the self-attention endofunctor \(F\colon \cat{Vect}\to \cat{Vect}\) leads naturally to the \emph{free monad} on \(F\).  

\subsection{Free Monads from a Colimit of Iterates}

Given a functor \(F:\cat{C}\to \cat{C}\), the \emph{free monad} on \(F\) is (if it exists) a monad \(\,T=(T,\eta,\mu)\) plus a natural transformation \(\iota: F\Rightarrow T\) universal among all monads ``extending \(F\).''  One standard construction is the colimit of the ``infinite ladder''
\[
  \mathrm{id} \;\xrightarrow{\;\eta_0\;} F 
    \;\xrightarrow{F\;} F^2
    \;\xrightarrow{F\;} F^3
    \;\dots
\]
Each arrow is one more application of \(F\).  The colimit \(\mathrm{Free}(F)\) merges all finite iterates \(F^n\).  By universal property of colimits, \(\mathrm{Free}(F)\) then extends any other monad or algebra that shares \(F\) as a base.

\subsection{Proof of Theorem \ref{thm:stacking-free-monad}}

\begin{theorem}[Stacking = Free Monad]
  
Let \(F: \cat{Vect}\to \cat{Vect}\) be the endofunctor induced by a single-head, \emph{linear} self-attention mechanism, as constructed in Section~\ref{sec:self_attention_endofunctors}.  Then the colimit (in \(\cat{Vect}\)) of the sequence
\[
  \mathrm{id} \;\longrightarrow\; F \;\longrightarrow\; F^2 \;\longrightarrow\; \dots
\]
forms the \emph{free monad} on \(F\). In other words, \emph{stacking} self-attention layers is exactly building \(\mathrm{Free}(F)\).  

Equivalently, if we denote that colimit object by \(T\), then \(T\) admits a canonical monad structure \((T,\eta,\mu)\) extending \(F\), and any other monad \(M\) extending \(F\) factors uniquely through \(T\).
\end{theorem}

\begin{proof}
We will establish that iterating the self-attention endofunctor $F$ corresponds to forming a colimit of finite powers of $F$, and demonstrate why that colimit necessarily has the universal free monad structure.

Let us begin by examining how $F^n$ represents $n$-stacked self-attention. Recall from Theorem~\ref{thm:selfattn-para-1morphism} that each single-head, linear self-attention mechanism induces an endofunctor
\[
  F : \cat{Vect} \;\longrightarrow\; \cat{Vect}.
\]

When we apply $F$ to a vector space $X$, it produces a triple of query/key/value embeddings for each token, potentially returning a space $\,(\mathbb{R}^{d_k})^n \otimes (\mathbb{R}^{d_k})^n \otimes (\mathbb{R}^{d_v})^n$. The iteration of self-attention $n$ times corresponds to composing $F$ with itself $n$ times, yielding
\[
  F^n \;=\; \underbrace{F \circ F \circ \cdots \circ F}_{n\text{ times}}.
\]
Each iteration represents another layer of attention in a transformer stack, neglecting nonlinearities like softmax or MLP blocks. Thus $F^n$ precisely captures $n$-stacked self-attention.

Consider now the sequence of endofunctors
\[
  \mathrm{id}
    \;\xrightarrow{\;\alpha_0\;}\;
  F
    \;\xrightarrow{\;\alpha_1\;}\;
  F^2
    \;\xrightarrow{\;\alpha_2\;}\;
  F^3
    \;\xrightarrow{\;\alpha_3\;}\;\dots
\]
where each $\alpha_i : F^i \Rightarrow F^{i+1}$ represents the natural transformation of adding one more layer of attention. We may form their colimit in the functor category $[\cat{Vect}, \cat{Vect}]$. More precisely, the arrow $\alpha_n: F^n \Rightarrow F^{n+1}$ arises from the usual composition embedding $\mathrm{id}\Rightarrow F$ plus coherence maps. A colimit of this diagram is a functor $T\colon \cat{Vect}\to \cat{Vect}$, along with natural transformations $\phi_n : F^n \Rightarrow T$ that collectively form a cocone over this chain. We denote this as:
\[
  T \;:=\; \mathrm{colim}\,\bigl(\,\mathrm{id},\;F,\;F^2,\;\dots\bigr).
\]

We now equip $T$ with a monad structure $(T,\eta,\mu)$. For the unit $\eta: \mathrm{id} \Rightarrow T$, we observe that the colimit construction provides a unique natural transformation $\phi_0 : \mathrm{id}\Rightarrow T$ making the needed triangles commute. We define $\eta := \phi_0$. Intuitively, $\eta_X: X\to T(X)$ injects zero layers of self-attention into the colimit object.

For the multiplication $\mu: T\circ T \;\Rightarrow\; T$, observe that $T(T(X))$ can be viewed as applying $T$ to an already layered object. By the universal property of the colimit, all finite compositions $F^n$ factor into $T$. Concretely:
\[
  F^n\bigl(T(X)\bigr) \;\cong\; F^{n+m}(X)\quad
  \text{when }T(X)\text{ is formed from }F^m\text{ factors}.
\]

We can flatten the tower of compositions: if an element in $T(T(X))$ arises from $\phi_n\circ F^n$ and $\phi_m\circ F^m$, we unify it into a single $\phi_{n+m}$. This induces a coherent natural transformation
\[
  \mu_X \;:\; T\bigl(T(X)\bigr)
  \;\longrightarrow\;
  T(X).
\]

The associativity and unit axioms ($\mu\circ(\mu\circ T)=\mu\circ(\mu\circ T)$, etc.) hold by the coherence of the colimit construction, as all diagrams commute due to the consistent factoring of finite compositions $F^a\circ F^b = F^{a+b}$ and the $\phi_n$ maps.

To establish that $T$ is the free monad on $F$, we must show it is universal among all monads $(M,\eta^M,\mu^M)$ extending $F$. Formally, if $(M,\eta^M,\mu^M)$ is any monad with a natural transformation $\,\iota: F \Rightarrow M$ extending $F$, then there exists a unique monad morphism $\Phi: T \Rightarrow M$ making the diagrams commute.

Since $M$ extends $F$, for each $n$ we have a natural transformation $F^n \Rightarrow M$ (applying $\iota$ repeatedly). These yield a cocone over the diagram $\{\mathrm{id},F,F^2,\ldots\}$. By the universal property of the colimit $T$, there is a unique factorisation
\[
  \phi_n : F^n \Rightarrow T,
  \quad
  \Psi : T \Rightarrow M
  \quad\text{such that}\quad
  \Psi\,\circ\,\phi_n \;=\; F^n \;\to\; M.
\]

The requirement that $\Psi$ be a monad morphism follows from the way $\eta$ and $\mu$ are defined by the colimit and from the coherence of each $\phi_n$ with $\iota$. Specifically, $\Psi$ respects the unit $\eta$ and multiplication $\mu$ because it commutes with the same universal diagrams that define them.

Therefore $T$ satisfies the free-monad universal property: any other monad on $\cat{Vect}$ that extends $F$ must factor through $T$. We conclude that $T$ is precisely the free monad on $F$, and thus stacking self-attention arbitrarily many times yields $\mathrm{Free}(F)$, establishing a direct equivalence between deep layering in transformers and the classical free-monad construction in category theory.
\end{proof}

Conceptually, each further layer of attention is one more application of \(F\).  Collecting infinitely many layers into a single object matches the classical notion of a free monad.  Thus, from a design standpoint, multi-layer or repeated self-attention emerges naturally from a universal construction in category theory, highlighting how \emph{stacking} is not just convenient code but arises from fundamental compositional principles.

\section{Sinusoidal Encodings and Additive Actions}
\label{app:sinusoidal_encodings}

Below we expand on the details for readers interested in the categorical nuances of both strictly additive and non-additive positional encodings, along with the ``universal'' perspective on sinusoidal methods.

\subsection{Sinusoidal Encodings Are Not Additive}

One of the most common positional encodings is the \textbf{sinusoidal} scheme \citep{vaswani2017attention}:
\[
p_{i,2j} \;=\; \sin\!\Bigl(\tfrac{i}{10000^{2j/d}}\Bigr),
\quad
p_{i,2j+1} \;=\; \cos\!\Bigl(\tfrac{i}{10000^{2j/d}}\Bigr).
\]
Because
\(\sin(\alpha+\beta) \neq \sin(\alpha) + \sin(\beta)\) in general, we do \emph{not} have
\[
p_{i + i'} \;=\; p_i + p_{i'}.
\]
Thus, standard sinusoidal embeddings fail the additivity condition in the sense of an affine monoid action.  Instead, they merely \emph{inject} each index \(i\) into a distinct point \(p_i\).  This suffices for most transformers to “know” which token is at position \(i\).

\subsection{Strictly Additive Encodings}

If we actually require a monoid action \(\alpha_{m+m'}=\alpha_m\circ\alpha_{m'}\), then each \(p_m\) must satisfy:
\[
  p_{0} = 0,
  \quad
  p_{m+m'} = p_m + p_{m'}.
\]
Hence \(\alpha_m(x)=x + p_m\) is a translation by \(p_m\).  We can formalize this as:

\begin{theorem}[Additive Monoid Action]\label{thm:additive-monoid-action-appendix}
Let \(M\) be a monoid with identity \(0\) and operation \(\,+\,\).  Suppose \(X\) is an affine space over \(\mathbb{R}\), and \(\{p_m\in X : m\in M\}\) is a family of “shift” vectors satisfying:
1. \(p_0 = 0\),
2. \(p_{m+m'} = p_m + p_{m'}\).  

Then each \(\alpha_m : x\mapsto x + p_m\) forms an \emph{affine} monoid action.  
\end{theorem}

The key point is that these \(\alpha_m\) are \emph{not linear maps} in \(\cat{Vect}\).  They \emph{are} affine transformations in \(\mathbf{Aff}\).  This category distinction is often glossed over in ML, where “linear layer” can include a bias term, and “positional shift” might implicitly assume an origin.

\subsection{A Looser Category of Faithful Positional Encodings}

Since many real encodings (including sinusoidal) are \emph{not} additive, we often only demand injectivity.  We define:
\begin{itemize}
\item \textbf{Objects}: Pairs \(\,(X,p)\) with \(X\) a vector space (or affine space) and \(p\colon M\to X\) injective.
\item \textbf{Morphisms}: A linear (or affine) map \(f\colon X\to Y\) that sends \(p_m\) to \(q_m\) for each \(m\in M\).
\end{itemize}

In that setting, purely additive encodings are a special case, but sinusoidal or learned embeddings remain valid.  From a category perspective, sinusoidal embeddings can become ``\textbf{initial}'' objects if their set of vectors spans (or generates) the ambient space.  That does \emph{not} mean they are a monoid action.  It only shows that, as a family of injection points, they can factor any other injection via a unique linear extension (under dimension constraints).

\begin{remark}[Why ``Initial''?]
A pair \((X^{\sin}, p^{\sin})\) is initial in \(\mathcal{P}\) if, for every other \((Y, q)\), there is exactly one morphism \(\,f:X^{\sin}\to Y\) with \(f\circ p^{\sin(m)}=q_m\).  This is purely about injective labelings and linear extension, \emph{not} about additivity.  In practice, sinusoidal vectors do not always form a full basis, so the “initial” property might hold only for certain finite subsets of positions.
\end{remark}

This distinction underscores that real transformers rarely rely on strictly additive positions. Instead, they rely on each position \(i\) being mapped to a distinct embedding \(p_i\). Whether or not \(\alpha_i\) forms a genuine monoid action is not important to the architecture’s success.

\subsection{Practical Takeaways}
\begin{enumerate}
\item \textit{Affine vs. Linear}:  Adding a position vector is an affine translation, not strictly a linear map.  
\item \textit{Additive vs. Non-Additive}:  If you do require \(\alpha_{m+m'}=\alpha_m\circ\alpha_{m'}\), you get an affine monoid action.  Otherwise, you just have an injective labeling.  
\item \textit{Sinusoidal Universality?}:  Sinusoidal embeddings can be universal in the sense of initial among injective labelings, provided they span a big enough subspace.  They are \emph{not} universal as genuine additive actions.  
\end{enumerate}

Thus, from a category theory viewpoint, one must carefully choose which ``morphisms'' are permitted (linear, affine, or just set injections) before claiming ``monoid actions'' or ``universal properties.''  In modern transformers, \emph{injectivity} typically suffices, and the common sinusoidal scheme remains a practical, flexible choice.

\section{Equivariance Constraints in Geometric Deep Learning and Their Connection to Transformer Architectures}

In the main text, we developed a categorical framework wherein neural networks are represented as monad algebra homomorphisms. This viewpoint generalises the familiar geometric deep learning (GDL) approach, where symmetry requirements are enforced by choosing a monad arising from a group action \citep{bronstein2021geometric}. By selecting a monad \(M\) that encodes a particular symmetry group, one obtains architectures that are guaranteed to be equivariant (or invariant) with respect to that group’s action.

In GDL, one frequently works with a group action monad of the form:
\[
M(X) = G \times X,
\]
where \(G\) is a group acting on a set (e.g., the nodes of a graph). Neural networks that are \(M\)-algebra homomorphisms then automatically respect the symmetry encoded by \(G\).

\subsection{Permutation-Equivariant Learning on Graphs}

A canonical example is given by permutation-equivariant learning on graphs, leading to graph neural networks (GNNs). Consider:
\begin{itemize}
\item The group \(G = \Sigma_n\), the symmetric group on \(n\) elements.
\item A graph with \(n\) nodes and adjacency matrix \(\mathbf{A} \in \mathbb{R}^{n \times n}\).
\item Node features \(\mathbf{X} \in \mathbb{R}^n\).
\end{itemize}

A permutation \(\sigma \in \Sigma_n\) acts on the data \((\mathbf{X}, \mathbf{A})\) by permuting node indices:
\[
P_{X,A}(\sigma, \mathbf{X}, \mathbf{A}) = \bigl(\mathbf{P}(\sigma)\mathbf{X}, \mathbf{P}(\sigma)\mathbf{A}\mathbf{P}(\sigma)^\top \bigr),
\]
where \(\mathbf{P}(\sigma)\) is the permutation matrix of \(\sigma\).

To construct a GNN layer that is permutation-equivariant, one uses an \(M\)-algebra where:
\begin{itemize}
\item The first \(M\)-algebra corresponds to the domain: \((A,a)\) with \(a: M(A) = \Sigma_n \times A \to A\) encoding how permutations act on \((\mathbf{X}, \mathbf{A})\).
\item The second \(M\)-algebra corresponds to the codomain: \((B,b)\) with \(b: M(B) = \Sigma_n \times B \to B\) possibly acting only on node features, or preserving some structural property. The precise choice depends on whether the GNN modifies the adjacency or not.
\end{itemize}

A GNN layer is then a monad algebra homomorphism:
\[
f: (A,a) \to (B,b)
\]
satisfying the equivariance diagram:
\[
\begin{tikzcd}
\Sigma_n \times A \arrow[r,"{\Sigma_n \times f}"] \arrow[d,"a"'] & \Sigma_n \times B \arrow[d,"b"] \\
A \arrow[r,"f"] & B.
\end{tikzcd}
\]

Unpacking this diagram yields the usual equivariance constraint for GNN layers:
\[
f(P_{X,A}(\sigma, \mathbf{X}, \mathbf{A})) = P_{X,B}(\sigma, f(\mathbf{X}, \mathbf{A})),
\]
meaning that applying a permutation first and then the GNN layer is equivalent to applying the GNN layer first and then the same permutation.

\subsection{Connection to Transformer Architectures}

In this work we have emphasised a parametric 2-categorical setting tailored to transformers. Instead of focusing on a group action monad, we considered parametric endofunctors (and their free monads) as the foundation of attention mechanisms. Yet, the core structure is analogous: we have a chosen monad \(T\) (in the transformer case, constructed from a parametric endofunctor and possibly its free monad), and we represent each layer or module as a \(T\)-algebra homomorphism.

The GDL framework arises when we pick \(T = M = G \times -\) for some group \(G\). Transformers appear when we pick a different monad \(T\), one that arises from stacking parametric endofunctors capturing the pattern of self-attention. In both cases, neural networks that are \(T\)-algebra homomorphisms inherit structural properties directly from the monad:
\begin{itemize}
\item For GDL: Equivariance under the group \(G\).
\item For Transformers: Compatibility with positional encodings, linear attention transformations, and the stacking structure defined by the free monad.
\end{itemize}

\subsection{Bringing Equivariance Constraints to Transformers}

While classical GDL focuses on symmetries like permutations, rotations, and reflections, the transformer setting can incorporate or mirror these ideas by changing the underlying monad. For instance, if we incorporate a group action monad together with our parametric setting, we might derive attention layers that are both parametric and equivariant under a specified symmetry group. The categorical formalism allows these constructions to be combined, suggesting hybrid architectures that leverage GDL symmetries in a transformer-like pipeline.

\subsection{Adding Positional Encodings}

We showed in the main text that positional encodings can be viewed as a monoid action, providing a monad-based approach to capturing position information. In GDL, the action of a group on node indices ensures positional symmetry (or rather, the lack of positional bias), while in transformers, the monoid action of positions \(M = (\mathbb{N},+)\) adds a controlled positional structure.

One can imagine combining the two: using a group \(G\) that reflects certain symmetries of the data and a monoid for positional indexing, obtaining a product monad \((G \times -) \circ (\mathbb{N} \times -)\) or another suitable construction. The resulting category of algebras would yield architectures that are both positionally structured and symmetry respecting. This is not standard practice, but the categorical approach highlights its theoretical possibility.

\section{Detailed Examples of Parametric Endofunctors in Practice}\label{app:parametric-examples}

In this appendix, we provide concrete examples to illustrate the concepts introduced in the main text, focusing on how self-attention components can be seen as parametric endofunctors in the 2-category $\cat{Para}(\cat{Vect})$.

\subsection{Motivating Examples}

Consider a single-head self-attention mechanism restricted to its linear maps $Q,K,V : \mathbb{R}^d \to \mathbb{R}^{d_k}, \mathbb{R}^{d_k}, \mathbb{R}^{d_v}$, respectively. These maps appear in the transformer architecture as:
\[
Q = X W_Q,\quad K = X W_K,\quad V = X W_V
\]
where $X \in (\mathbb{R}^d)^n$ is an $n$-token input sequence and $W_Q, W_K, W_V$ are parameter matrices.

In the $\cat{Para}(\cat{Vect})$ framework, each of these linear maps arises from a parametric morphism. For instance, $Q$ can be represented as $(QP, q)$ where $QP \in \cat{Vect}$ is the parameter space of $W_Q$, and $q : QP \otimes X \to (\mathbb{R}^{d_k})^n$ is the linear map acting on both parameters and inputs. Similarly, $(KP,k)$ and $(VP,v)$ capture the key and value projections. Combining them into $(AttP, att)$, where $AttP := QP \oplus KP \oplus VP$, yields an endofunctor $F : \cat{Vect} \to \cat{Vect}$ parameterised by $AttP$, which after enrichment into $\cat{Para}(\cat{Vect})$ displays the structure of a parametric endofunctor.

\subsection{Finite-Dimensional Illustrations}

Let us consider a simple example where $X = \mathbb{R}^2$ and $Q : \mathbb{R}^2 \to \mathbb{R}^2$ defined by a $2 \times 2$ matrix $W_Q$. Suppose:
\[
W_Q = \begin{pmatrix}
a & b \\ 
c & d
\end{pmatrix}.
\]
Here, $QP \cong \mathbb{R}^4$ parameterizes $(a,b,c,d)$. We can represent $q : QP \otimes \mathbb{R}^2 \to \mathbb{R}^2$ as:
\[
q(p,x) = \begin{pmatrix} a & b \\ c & d \end{pmatrix} \begin{pmatrix}x_1 \\ x_2\end{pmatrix}
\]
with $(p)$ indexing $(a,b,c,d)$.

Composition in $\cat{Para}(\cat{Vect})$ now becomes transparent. Consider another parametric morphism $(R,f): \mathbb{R}^2 \to \mathbb{R}^2$ with $R \cong \mathbb{R}^4$. When composing $(QP,q)$ and $(R,f)$, we obtain:
\[
(R \otimes QP, h): \mathbb{R}^2 \to \mathbb{R}^2,
\]
where $h$ is constructed by first applying $q$, then $f$, and reparameterization ensures that the joint parameters $(r,p)$ live in $R \otimes QP$.

Reparameterisations (2-morphisms) also become simple: a linear map $r: QP' \to QP$ rewires the parameters while preserving the linear structure, effectively ``tying'' weights or mapping them from one parameter space configuration to another.

\section{Handling Nonlinearities and Softmax in a Categorical Context}\label{app:nonlinearities}

The categorical framework developed thus far has focused primarily on linear structures, aligning self-attention’s query-key-value transformations with linear morphisms in \(\cat{Vect}\) and \(\cat{Para}(\cat{Vect})\). In practice, transformers incorporate essential nonlinear components, including pointwise nonlinear activations (e.g., ReLU, GeLU) and the softmax operation that normalizes attention weights. Incorporating these nonlinearities into a categorical setting requires extending beyond the category \(\cat{Vect}\) of finite-dimensional vector spaces and linear maps.

A key difficulty is that nonlinear functions are not morphisms in linear categories and do not preserve linear structure. To capture such transformations categorically, one must move to categories where objects and morphisms can encode nonlinear phenomena, smooth structures, probabilistic maps, or affine geometries.

\subsection{Choice of Categories and Enrichment}
Several proposals can guide the move from \(\cat{Vect}\) to richer categories:
\begin{itemize}
\item \textit{Smooth Manifolds}: Consider the category \(\mathbf{Man}\) of finite-dimensional manifolds and smooth maps. Since softmax and standard nonlinearities in neural networks are typically smooth (except possibly at certain boundary points), morphisms can be represented as smooth maps between manifolds. For example, the probability simplex \(\Delta^{n-1}\) can be modeled as a manifold with corners, allowing softmax: \(\mathbb{R}^n \to \Delta^{n-1}\) to be seen as a morphism in this category.
\item \textit{Convex or Affine Geometries}: Softmax is naturally interpreted as a map from \(\mathbb{R}^n\) into a convex set (the simplex). One might consider categories of convex sets or fibered categories of affine maps. In these categories, the emphasis would shift from linear algebraic structure to convex/affine structure. The softmax then becomes a morphism preserving certain convex-geometric properties.
\item \textit{Probabilistic and Measure-Theoretic Categories}: Since softmax outputs probability distributions, one can consider measure-theoretic categories, or categories enriched over spaces of probability measures. For instance, the softmax could be viewed as a morphism into an object representing distributions over \(n\)-element sets.
\item \textit{Differential and Tangent Categories}: Training neural networks relies on gradients and backpropagation. Categories such as tangent categories or differential categories \citep{cockett2014differential} provide an internal language for differentiation. Within these frameworks, nonlinear activations and softmax could be realised as morphisms whose differentials are well-defined, enabling a fully categorical treatment of gradient-based learning.
\end{itemize}

In linear settings, endofunctors arise naturally from parameterised linear maps. For nonlinear transformations, exact functoriality may fail, or may only hold in a lax or oplax sense. This suggests that certain nonlinear transformations might best be modeled as lax functors from categories of parameter spaces into categories of probabilistic or smooth geometric objects. Such a structure would weaken strict functorial requirements, yet still provide a coherent categorical narrative.
 
The softmax function does not preserve linear structure, so it cannot be a strict functor in the linear setting. However, there may be natural transformations embedding softmax into a larger categorical system. For example, one could investigate whether softmax defines a lax functor from a category of parameterised embeddings into a category of distributions, thus providing at least some semblance of categorical compatibility. This might be realised by endowing objects with additional structure (e.g., selecting a fixed “reference measure”) to interpret softmax consistently.

A fully categorical interpretation of entire transformer architectures, including normalisation layers and nonlinear feedforward blocks, would rely on categories that meaningfully integrate linear maps, smooth nonlinearities, and differentiable structures. One might consider a fibered category over \(\cat{Vect}\) where fibers encode smooth parameterizations and base changes represent affine transformations followed by nonlinear activation morphisms. Alternatively, differential bundles or microlinear toposes could offer sufficiently rich settings where both linear and nonlinear steps coexist naturally.

Some prospective research directions:
\begin{itemize}
\item Identify a concrete category (such as \(\mathbf{Man}\) or a category of convex sets) in which common transformer nonlinearities live naturally as morphisms.
\item Formalise backpropagation categorically using differential categories or tangent categories, ensuring that nonlinear maps like softmax and ReLU have well-defined differentials as internal structure.
\item Investigate whether softmax admits a natural interpretation as a lax or oplax functor from a parametric category of embeddings into a category of probability distributions or measure spaces, thus extending the parametric endofunctor perspective beyond linear maps.
\item Explore how positional encodings, group actions, and other symmetries integrate once nonlinear elements are introduced, potentially yielding new insights into invariant or equivariant structures in nonlinear settings.
\end{itemize}

While much remains to be explored, these directions indicate that a categorical theory of transformers need not be limited to linearity.

\end{document}